\documentclass[letterpaper]{article} 
\usepackage{aaai25} 
\usepackage{times} 
\usepackage{helvet} 
\usepackage{courier} 
\usepackage[hyphens]{url} 
\usepackage{graphicx} 
\usepackage{graphicx} 
\usepackage{natbib} 
\usepackage{caption} 
\usepackage{xr, xcolor}
\usepackage{enumitem}
\usepackage{booktabs, multirow}
\usepackage{makecell}
\usepackage{tabularx}
\usepackage{subfig}
\usepackage{algpseudocode}
\usepackage[linesnumbered,lined,ruled]{algorithm2e}
\usepackage{tikz}
\usetikzlibrary{arrows.meta, matrix, positioning, fit, calc}
\usepackage{amsthm, amsmath, amsopn}

\usepackage{xspace}
\newcommand{\sys}{\mbox{\textsc{CCNF}}\xspace}

\newcommand{\TODO}[1]{\textcolor{Melon}{TODO: #1}}

\newcommand{\MX}[1]{\textcolor{Magenta}{[MX: #1]}}

\usepackage{xstring}
\newcommand{\PP}[1]{
\vspace{2px}
\noindent{\bf \IfEndWith{#1}{.}{#1}{#1.}}
}

\usepackage{relsize}
\newcommand{\cc}[1]{\mbox{\smaller[0.5]\texttt{#1}}}

\usepackage{array}
\usepackage{newtxtext,newtxmath}
\newtheorem{theorem}{Theorem}[section]

\newtheorem{lemma}[theorem]{Lemma}
\newcolumntype{L}{>{\raggedright\arraybackslash}X}

\usepackage{tikz}
\newcommand*\WC[1]{%
\begin{tikzpicture}[baseline=(C.base)]
\node[draw,circle,inner sep=0.2pt](C) {#1};
\end{tikzpicture}}

\title{Causally Consistent Normalizing Flow}
\author {
    Qingyang Zhou\textsuperscript{\rm 1},
    Kangjie Lu\textsuperscript{\rm 2},
    Meng Xu\textsuperscript{\rm 1}
}
\affiliations {
    \textsuperscript{\rm 1}University of Waterloo, Ontario, Canada\\
    \textsuperscript{\rm 2}University of Minnesota, Minnesota, America\\
    qingyang.zhou@uwaterloo.ca, 
    kjlu@umn.edu,
    meng.xu.cs@uwaterloo.ca
}

\begin{document}

\maketitle
\begin{abstract}
Causal inconsistency arises when the underlying causal graphs captured by generative models like \textit{Normalizing Flows} (NFs) are inconsistent with those specified in causal models like \textit{Struct Causal Models} (SCMs).
This inconsistency can cause
unwanted issues including the unfairness problem.
Prior works to achieve causal consistency inevitably compromise the expressiveness of their models
by disallowing hidden layers.
In this work, we introduce a new approach: \textbf{C}ausally \textbf{C}onsistent \textbf{N}ormalizing \textbf{F}low (\sys).
To the best of our knowledge, 
\sys is the first causally consistent generative model 
that can approximate any distribution with multiple layers.
\sys relies on two novel constructs:
a sequential representation of SCMs and
partial causal transformations.
These constructs allow \sys
to inherently maintain causal consistency
without sacrificing expressiveness.
\sys can handle all forms of causal inference tasks, 
including interventions and counterfactuals.
Through experiments,
we show that \sys outperforms current approaches 
in causal inference.
We also empirically validate the practical utility of \sys 
by applying it to real-world datasets
and show how \sys
addresses challenges like unfairness effectively.
\end{abstract}

\section{Introduction}
\label{s:intro}

Causal generative modeling is generative models (GMs) that utilize
given causal models like structure causal models (SCMs)
for data generation{\small ~\cite{vaefuture}}.
It has been widely researched on 
different types of GMs
like VAE{\small~\cite{causalvae}}, 
GAN{\small~\cite{causalgan}}, 
Normalizing Flow (NF){\small~\cite{causalnf}} 
and Diffusion Model{\small~\cite{difscm}}.

However, most approaches have a problem that they can only approximate 
the causality relations
instead of \emph{enforcing} the consistency
between the causal graph induced by GMs and
the causal graph in given SCMs.
The problem is called 
casual inconsistency problem in prior works{\small~\cite{causalnf}}
and will be discussed in detail in Section~\ref{s:background}.
This could lead to critical societal issues 
(e.g. the one in Section~\ref{s:moti}),
which have yet to be adequately addressed.

\begin{figure*}[!t]
    \centering
    \subfloat[causal NF]{
       \includegraphics[width=2.5in]{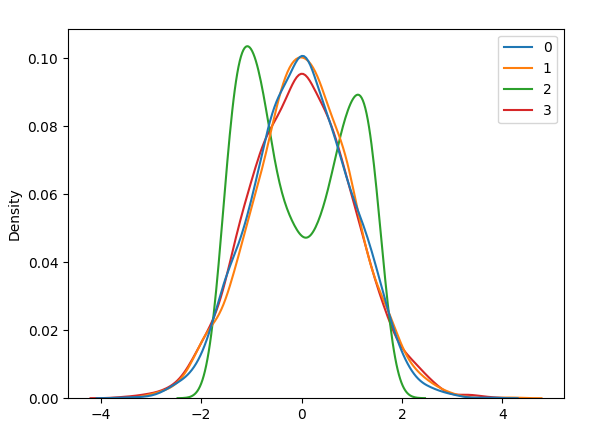}
       \label{fig:prior.a}
    }
    \hfil
    \subfloat[\sys]{
        \includegraphics[width=2.5in]{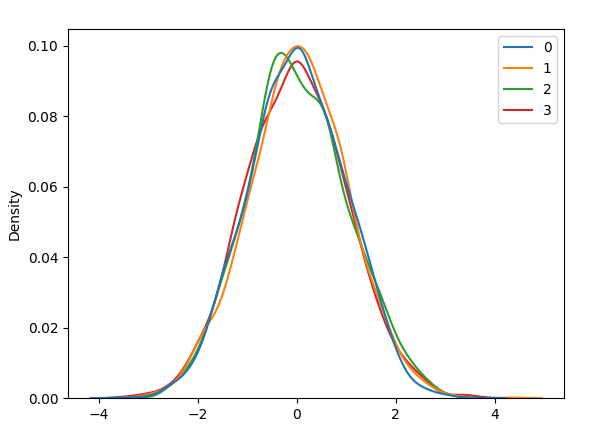}
      \label{fig:prior.b}
    }
    \caption{Prior distributions of causal NF and \sys}
    \label{fig:prior}
\end{figure*}

Fortunately, 
recent works have proposed GMs 
that are causally consistent by design.
%
Causal NF{\small~\cite{causalnf}} and VACA{\small~\cite{vaca}}
ensure causal consistency
by limiting their models to minimal structural complexity.
%
More specifically,
Causal NF restricts the model depth to zero,
i.e., eliminating middle layers in NF;
%
while
VACA applies a similar restriction to its encoder structure.
%
In other words,
causal consistency is guaranteed at the expense of
the utility of these models---the ability to approximate
any arbitrarily complex distributions of observations.
%
%
For instance,
the training objective of Causal NF is 
to minimize the discrepancy
between the distributions of latent variables 
and the pre-selected distributions by users
(e.g. Gaussian).
%
However, as shown in Figure~\ref{fig:prior.a}, 
a Causal NF trained on a nonlinear Simpson dataset
is not able to accomplish the objective.
The distribution of the third latent variable, 
highlighted in green in Figure~\ref{fig:prior.a}, 
deviates significantly from the Gaussian distribution.
%
%

%
In this paper,
we introduce
\textbf{C}ausally \textbf{C}onsistent \textbf{N}ormalizing \textbf{F}low, 
abbreviated as \sys,
that is a causally consistent GM by design
without sacrificing utility,
i.e., approximating arbitrarily complex
distributions based on universal approximation theorems
(details in Theorem~\ref{universality}).
%
A key innovation of \sys is
to translate an SCM
into a sequence 
(details in Secion~\ref{s:design}).
%
%
%
%
The sequential representation of an SCM
eliminates the constraints of
maximum layer depths
without compromising causal consistency,
enabling a more flexible model architecture in \sys
(details in Section~\ref{s:operations}).
%
%
%
%
Subsequently,
\sys employs normalizing flows 
with partial causal transformations 
to effectively capture the causality in the data, 
which is further elaborated in Section~\ref{s:design} and Section~\ref{s:operations}.

We demonstrate that \sys 
is inherently causally consistent
and capable of performing causal inference tasks
such as interventions and counterfactuals.
%
To the best of our knowledge,
\sys is the first causally consistent GM that can
approximate any distributions of observation variables
across multiple layers.
In comparison,
\sys outperforms existing models like Causal NF
in similar tasks,
as shown in~\ref{fig:prior.b}.
Additionally, 
\sys proves effective in real-world applications,
addressing significant issues such as unfairness.
%
Applying \sys to the German credit dataset{\small~\cite{german}}, 
we observe notable improvements:
a reduction in individual unfairness from 9.00\% to 0.00\%,
and an increase in overall accuracy from 73.00\% to 75.80\%.
%

\PP{Summary}
This paper makes the following contributions:
\begin{itemize}[align=left,leftmargin=*]
  \item We propose a new sequential representation for SCMs,
  and formally prove its ability
  to maintain the causal consistency.
  
  \item Utilizing this sequential representation
  alongside partial causal transformations,
  we develop \sys, 
  a GM that guarantees causal consistency
  and excels at complex causal inference tasks.
  %

  \item We empirically validate \sys,
  demonstrating it outperforms state-of-the-art
  casually consistent GMs
  on causal inference benchmarks. 
  Furthermore,
  our real-world case study showcase
  the potential of \sys in addressing critical issues like unfairness.
  \sys is open-sourced in the code link.
\end{itemize}

\section{A Motivating Example}
\label{s:moti}

To articulate the  
causal inconsistency problem between GMs and SCMs,
we present 
a motivating example modeling a simplified admission system.

\PP{A simplified admission system}
The causal graph of
a simplified admission system \(\mathcal{M}\)
is shown in Figure~\ref{fig:admission}.
It consists of four attributes:
\cc{gender},
\cc{age},
\cc{score}, and
admission \cc{decision}.
In terms of casualty,
\cc{gender} and \cc{age} determine the distribution of \cc{score},
and ideally, \cc{score} solely determines
the distribution admission \cc{decision},
ensuring that \cc{gender} does not
(and should not)
directly affect admission.
%

To illustrate,
we assume the observations \(\mathbf{O}\)
of this admission system \(\mathcal{M}\)
can be generated
by the equations in Table~\ref{tab:comp_equation} under 
the SCM column.
Here,
the value of independent variables $u_i$
where $i \in \{g, a, s, d\}$
are randomly sampled from predefined distributions.
For instance,
\(u_g\) is sampled from a distribution
where gender is distributed equally.
With distinct samples for $u_i$ where $i \in \{g, a, s, d\}$,
this system can generate unique observations representing \(\mathbf{O}\).

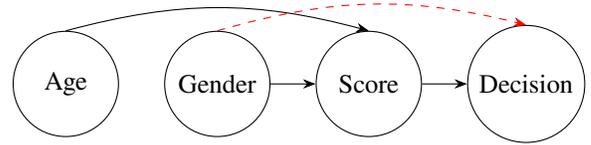
\begin{figure}[!t]
    \centering
    \begin{tikzpicture} [
node/.style = {circle, draw, minimum size = 1.4cm},
node distance = 0.6cm 
]

\node[node] (age) {Age};
\node[node] (gender) [right=of age] {Gender};
\node[node] (score) [right=of gender] {Score};
\node[node] (decision) [right=of score] {Decision};

\draw[-Stealth] (gender.east) -- (score.west);
\draw[-Stealth] (score.east) -- (decision.west);
\draw[-Stealth, bend left = 15] (age.north) to (score.north)[above];
\draw[red, dashed, -Stealth, bend left = 15] (gender.north) to (decision.north)[above];
\end{tikzpicture}
    \caption{The causal graph of an SCM
    describing an admission system
    with direct causalities that are intended (black solid line) and
    forbidden (red dashed line)}
    \label{fig:admission}
\end{figure}

\begin{table}[ht]
    \centering
    \setlength{\tabcolsep}{1mm}
    {\small
    \begin{tabular}{c|cc}
    \toprule
    Variable & SCM & GM \\
    \midrule
    \cc{gender} & $u_g$ & $u_g$ \\
    \cc{age} & $u_a$ & $u_a$ \\
    \cc{score} &  $u_s$+$\cc{gender}$+$\cc{age}$ &  $u_s$-$\cc{gender}$+$\cc{age}$ \\
    \cc{decision} & $f(u_{d}$+$\cc{score})$ & \makecell{$f(u_d$+$\cc{score}$+2*$\cc{gender})$}  \\
    \bottomrule
    \end{tabular} 
    }
    \caption{Comparison of generative equations between the actual SCM and GM. The function $f$ in the table stands for the $sign$ function.}
    \label{tab:comp_equation}
\end{table}

\PP{GMs based on the admission system}
In this scenario, \(\mathcal{GM}\)s learned from the observations \(\mathbf{O}\)
offer extensive capabilities{~\small~\cite{gmsurvey}}. 
However,
a \(\mathcal{GM}\) may establish incorrect causal links,
such as between \cc{gender} and admission \cc{decision}, 
as indicated by the red dashed line in Figure~\ref{fig:admission}.
%
%
This incorrect causality suggests that
in \(\mathcal{GM}\), 
different values of \cc{gender} could lead to 
varied distributions of admission \cc{decision}, 
even if the \cc{score} is identical.

Consider the scenario 
where the underlying causal relationships of
a \(\mathcal{GM}\) are formulated as in the GM column of the Table~\ref{tab:comp_equation}.
We can verify that 
the data distributions generated by the \(\mathcal{GM}\) is
indistinguishable from those of \(\mathcal{M}\),
despite in distinct functional forms.
%
However,
\(\mathcal{GM}\) and \(\mathcal{M}\) are not
causally consistent, which 
could result in significant consequences.

\PP{Unfairness due to causal inconsistency}
In this example,
causal inconsistency
could cause an unfairness problem.
%
Consider a scenario where users generate data instances
with the same $\cc{score}$.
In the origin system \(\mathcal{M}\),
as \cc{gender} has no direct impact on admission \cc{decision},
users will observe that 
changing the gender attribute 
does not alter the distribution of admission \cc{decision}.
However,
in the causally inconsistent model \(\mathcal{GM}\),
users will surprisingly observe that
different values of \cc{gender} could lead to
different distributions of \cc{decision},
even with the same score.
Such gender bias could lead to intense social debates
and the organization deploying the \(\mathcal{GM}\)
might face legal challenges{\small ~\cite{facebook}}.

\begin{figure}[!t]
    \centering
    \subfloat[Consistent] {
       \begin{tikzpicture} [
    node distance=0.4cm, 
    every matrix/.style = {
        matrix of nodes,
        nodes={
            circle,
            minimum size = 0.9cm
        }, 
        row sep=0.1cm
    }
]
    \matrix (m1) {
        |[draw]| $U_g$ \\
        |[draw]| $U_a$ \\
        |[draw]| $U_s$ \\
        |[draw]| $U_d$ \\
    };
    \node[above] at (m1.north) {$U$};

    \matrix (m2) [right = of m1] {
        |[draw]| $G$ \\
        |[draw]| $A$ \\
        |[draw]| $S$ \\
        |[draw]| $D$ \\
    };
    \node[above] at (m2.north) {$X$};

    \draw[-] (m2-1-1) -- (m1-1-1);
    \draw[-] (m2-2-1) -- (m1-2-1);
    \draw[-] (m2-3-1) -- (m1-3-1);
    \draw[-] (m2-4-1) -- (m1-4-1);
    \draw[-] (m2-1-1) -- (m1-3-1);
    \draw[-] (m2-2-1) -- (m1-3-1);
    \draw[-] (m2-3-1) -- (m1-4-1);
\end{tikzpicture}
       \label{fig:priormoti.cc}
    }
    \hfil
    \subfloat[Inconsistent] {
       \begin{tikzpicture} [
    node distance=0.4cm, 
    every matrix/.style = {
        matrix of nodes,
        nodes={
            circle,
            minimum size = 0.9cm
        }, 
        row sep=0.1cm
    }
]
    \matrix (m1) {
        |[draw]| $U_g$ \\
        |[draw]| $U_a$ \\
        |[draw]| $U_s$ \\
        |[draw]| $U_d$ \\
    };
    \node[above] at (m1.north) {$U$};

    \matrix (m2) [right = of m1] {
        |[draw]| $M$ \\
        |[draw]| $M$ \\
        |[draw]| $M$ \\
        |[draw]| $M$ \\
    };
    \node[above] at (m2.north) {$M$};
     
    \matrix (m3) [right = of m2] {
        |[draw]| $G$ \\
        |[draw]| $A$ \\
        |[draw]| $S$ \\
        |[draw]| $D$ \\
    };
    \node[above] at (m3.north) {$X$};

    \draw[-] (m2-1-1) -- (m1-1-1);
    \draw[-] (m2-2-1) -- (m1-2-1);
    \draw[-] (m2-3-1) -- (m1-3-1);
    \draw[red, dashed, -] (m2-4-1) -- (m1-4-1);
    \draw[-] (m2-1-1) -- (m1-3-1);
    \draw[-] (m2-2-1) -- (m1-3-1);
    \draw[red, dashed, -] (m2-3-1) -- (m1-4-1);

    \draw[-] (m3-1-1) -- (m2-1-1);
    \draw[-] (m3-2-1) -- (m2-2-1);
    \draw[-] (m3-3-1) -- (m2-3-1);
    \draw[red, dashed, -] (m3-4-1) -- (m2-4-1);
    \draw[red, dashed, -] (m3-1-1) -- (m2-3-1);
    \draw[-] (m3-2-1) -- (m2-3-1);
    \draw[-] (m3-3-1) -- (m2-4-1);
\end{tikzpicture}
       \label{fig:priormoti.cic}
    }
    \caption{Causally consistent models and inconsistent models of prior works. \textit{G} = \textit{Gender}, \textit{A} = \textit{age}, \textit{S} = \textit{Score}, \textit{D} = \textit{Decisions}, \textit{M} stands for nodes in the middle layer.}
    \label{fig:priormoti}
\end{figure}
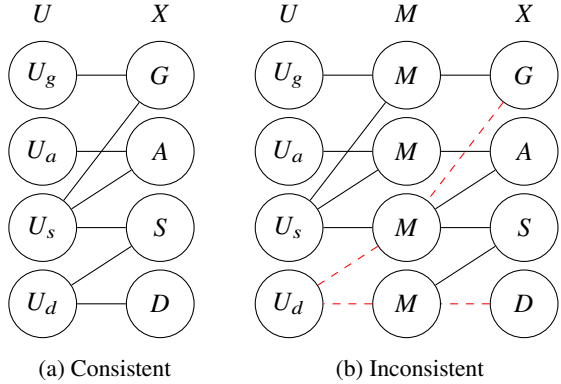

\PP{Prior works and their drawbacks}
Despite being applied to various architectures of GMs,
prior works on guaranteed causally consistent GMs
share a common intuition:
%
%
while each attribute in observation is influenced
only by its parent attributes
within the causal graph,
the causality relation must be captured
in the GM model \textbf{all at once},
which means that their layer depth must be zero.
%
Figure~\ref{fig:priormoti.cc} shows an example of this approach.
In this scenario,
if we maintain constant values for \cc{age} and \cc{score}
and only mutate the value of \cc{gender},
The distribution of
admission \cc{decision} should remain unaffected.

This intuition indeed ensures causal consistency,
but it has some shortcomings;
notably, it does not allow for
incorporation of any middle layers.
For example,
mutating the value of \cc{gender}
in Figure~\ref{fig:priormoti.cic}
results in a change of admission \cc{decision} 
due to the red dashed connections. 
%
This comparison between Figure~\ref{fig:priormoti.cic} 
and Figure~\ref{fig:priormoti.cc}
illustrates how causal consistency is compromised 
even with the introduction of a single layer in the middle,
while on the other hand,
forbidding middle layers
significantly impairs the model's capacity 
for learning.
%
However,
as we will demonstrate later,
\sys can maintain causal consistency
even with multiple middle layers.
Therefore \sys offers great learning ability 
compared to previous models, 
enhancing practical applicability.
\section{Preliminaries}
\label{s:background}

In this section,
we define basic concepts and related lemmas
to set the stage for \sys.
All definitions and lemmas introduced in this section
are consistent with prior works{\small~\cite{carefl, nf, dagnn}}.

\subsection{Structured Casual Model (SCM)}

\PP{Definition} 
A structural causal model (SCM) is a tuple \(\mathcal{M} = (\widetilde{\mathbf{f}}, P_\mathrm{u})\) commonly used to represent causality.
It describes the process that a set of \(d\) endogenous (observed) random variables \(\mathbf{X} = \{X_1, \cdots, X_d\}\) is generated from a corresponding set of exogenous (latent)
random variables \(\mathbf{U} = \{U_1, \cdots, U_d\}\) 
associated with a set of predefined
distributions \(P_\mathrm{u}\) and
a set of transfer functions \(\widetilde{\mathbf{f}}\).
Typically, \(\mathbf{X}\) and  \(\mathbf{U}\) have the same length, denoting as \(d\).
The generation of \(\mathbf{X}\) is governed by the equation:
\begin{equation}
    \label{eq:scm}
    X_i = \widetilde{f_i}(\mathbf{X}_{\mathbf{pa}_i}, U_i), \quad i \in \{1, \ldots, d\}
\end{equation}
where \(\mathbf{X}_{\mathbf{pa}_i}\) denotes a set of endogenous variables that \emph{directly} influences \(X_i\), i.e. the parents of \(X_i\).
Particularly, \(\mathbf{pa}_i\) represents the labels of those parent variables of \(X_i\).
%
%
%
%
Generally, we assume exogenous variables \(\mathbf{U}\) are mutually independent.

\PP{Graphical representation of an SCM: causal graph}
A causal graph \(\mathcal{G} = (\mathbf{V}, \mathbf{E})\) with \(|\mathbf{V}| = d\) nodes representing an SCM is a powerful tool to describe causality.
Each node in \(\mathbf{V}\) corresponds to an endogenous random variable \(X_i\), and
each edge in \(\mathbf{E}\) represents a causal link from a variable in \(\mathbf{X}_{\mathbf{pa}_i}\) to \(X_i\).
As a common assumption{\small~\cite{pearl}}, the causal graph \(\mathcal{G}\) is structured as a directed acyclic graph (DAG).
For convenience, each node \(V \in \mathbf{V}\) is assigned with a label \(i \in \mathbf{L} = \{1, \cdots, d\}\), and we denote it as \(V_i\).
For a subset of labels \(\mathbf{A} \subseteq \mathbf{L}\),
we define \(\mathbf{V_A} = \{V_i \mid i \in \mathbf{A} \}\).
See Figure~\ref{fig:topo} for an illustration.

\PP{Causal inference and causal hierarchy}
\textit{Causal inference} denotes the data generation process by SCM.
According to the \textit{causal hierarchy}{\small~\cite{pearl}},
the generative process is classified into three distinct tiers: 
\textit{observations},
\textit{interventions}, and
\textit{counterfactuals}.
The \textit{observations} process involves generating
\(\mathbf{X}\) unconditionally.
This is straightforward:
we generate \(\mathbf{U}\) sampled from 
the predefined distributions \(P_\mathrm{u}\) 
and then compute \(\mathbf{X}\) from \(\mathbf{U}\) by the formula in Equation~\ref{eq:scm}.
The \textit{interventions} process involves generating 
\(\mathbf{X}\) while setting the \(X_i\) 
to a specific value of \(a\),
often represented as \(Do(X_i = a)\).
This requires modifying the SCM 
such that every \(X_i\) in Equation~\ref{eq:scm} 
is replaced with \(a\) to create a new SCM:
\(\mathcal{M}_{Do(X_i = a)}\).
\(\mathbf{X}\) is generated by performing observations on the new SCM.
The \textit{counterfactuals} process considers a specific data instance \(\mathbf{X}\) where \(X_j = b\),
and aims to generate data instances 
supposing that \(X_j = b' \neq b\).
This process first deduces \(\mathbf{U}\) from \(\mathbf{X}\) 
via Equation~\ref{eq:scm},
then performs \(Do(X_j = b')\) to create a new SCM \(\mathcal{M}_{Do(X_j = b')}\),
Subsequently, the data is generated by
inputting \(\mathbf{U}\) into \(\mathcal{M}_{Do(X_i = a)}\).

\subsection{Causal Normalizing Flows}
\PP{Normalizing flows}
\textit{Normalizing flows} constitute a set of generative models that express the probability of observed variables \(\mathbf{X}\) from \(\mathbf{U}\) by change-of-variables rules.
Particularly, given \(\mathbf{X} = \{X_1, \cdots, X_d\}\) and \(\mathbf{U} = \{U_1, \cdots, U_d\}\), the probability of \(\mathbf{X}\) is expressed as follows:
\begin{gather}
    \mathbf{X} = T_{\theta}(\mathbf{U}), \quad where \   \mathbf{U} \sim P_\mathbf{U}\label{eq:trans} \\
    P_{\mathbf{X}}(\mathbf{X}) = P_{\mathbf{U}}(T^{-1}_{\theta}(\mathbf{X}))|\det{J_{T^{-1}_{\theta}}}(\mathbf{X})|
\end{gather}
Here, \(T_{\theta}\) represents a transformation that maps endogenous variables \(\mathbf{U}\) to exogenous variables \(\mathbf{X}\).
\(T_{\theta}\) could be any transformation as long as it is a partial derivative and invertible, often realized by a neural network with parameter \(\theta\).
It is common to chain different transformations \(T_{\theta_1} \cdots T_{\theta_k}\) to form a larger transformation \(T_{\theta} = T_{\theta_k} \circ \cdots T_{\theta_2} \circ T_{\theta_1}\).
\(P_{\mathbf{X}}\) denotes the probability of \(\mathbf{X}\) while \(P_{\mathbf{U}}\) denotes the probability of \(\mathbf{U}\), 
where \(P_{\mathbf{X}}\) is the target and \(P_{\mathbf{U}}\) usually follows a simple distribution.
The \(\det{J}\) means the Jacobian determinant of a given function, which is \(T^{-1}_{\theta}\) in this formula.

\PP{Multi-layer universal approximator}
An NF serves as a \textit{multi-layer universal approximator}, implying that any \(P_{\mathbf{X}}\) can be approximated by chaining a finite number of transformations.
Comparatively, the single-layer universal approximator can achieve the same objective with only one transformation.
The assumption that a NF is single-layer universal is stronger than the assumption that it is multi-layer.

Although an NF the theoretical capability is single-layer universal{\small~\cite{nf}},
No concrete NF succeeded in proving this.
Instead, 
a recent study on
the university of coupling-based NF 
proves 
that affine coupling flows like MAF{\small~\cite{maf}}
are multi-layer universal{\small~\cite{cnfuniversality}}.

\PP{Autoregressive normalizing flows}
\textit{Autoregressive normalizing flows} are a type of NFs whose transformations are defined as below:
given two random variables \(\mathbf{U} = \{U_1, \cdots, U_d\}\) and \(\mathbf{X} = \{X_1, \cdots, X_d\}\), the special transformation of Equation ~\ref{eq:trans} is:
\begin{equation}
    \label{eq:auto}
    X_i = T_{\theta}(U_i \mid \mathbf{X}_{<i}), \quad i \in \{1, \ldots, d\}
\end{equation}
Here, \(\mathbf{X}_{<i} = \{X_1, \cdots, X_{i-1}\}\).
This formula indicates that the parameter \(\theta\) in \(T_{\theta}(U_i \mid \mathbf{X}_{<i})\) is determined by \(\mathbf{X}_{<i}\), and the output value \(X_i\) is directly determined only by \(U_i\).
This transformation yields a simple Jacobin determinant since its Jacobin matrix is lower triangular.

\PP{Autoregressive flows and causality}
Although many works{\small~\cite{deepscm, imagepcm}} have utilized NFs for causal inference tasks
like counterfactual inference,
autoregressive flows and casualty were still
considered as two unrelated fields.
However, Ilyes noticed that 
it is possible to leverage autoregressive flows 
for causal tasks due to their intrinsic similarity{\small~\cite{carefl}}.
Particularly, to capture the causal relation between \(X_i\) and \(U_i\) accurately, 
autoregressive flows
possess a transformation as follows:
\begin{equation}
    \label{eq:causalux}
    X_i = T_{\theta}(U_i \mid \mathbf{X}_{\mathbf{pa}_i}), \quad i \in \{1, \ldots, d\}
\end{equation}
In contrast to Equation~\ref{eq:auto}, the parameter \(\theta\) in Equation~\ref{eq:causalux} is determined by \(\mathbf{X}_{\mathbf{pa}_i}\) rather than \(\mathbf{X}_{<i}\).

\subsection{Causal Consistency}
For any given SCM \(\mathcal{M}\) and 
the its casual graph \(\mathcal{G_M}\),
we call a \(\mathcal{GM}\) is 
causally consistent with \(\mathcal{M}\)
if the causal graph \(\mathcal{G_{GM}}\) 
induced by the \(\mathcal{GM}\)
is the same as \(\mathcal{G_M}\).
According to previous work,
the special \(\mathcal{GM}\)s exist
and can produce consistent result 
in all three tiers of causal hierarchy{\small~\cite{nerualcm}}.

Note that causal consistency only
require the \(\mathcal{M}\) and \(\mathcal{G_M}\)
to share the same causal graph.
In previous work{\small~\cite{indeterminacy}},
they have proved that 
such \(\mathcal{GM}\) and \(\mathcal{M}\) are identifiable,
which means
the data-generating process
between \(\mathcal{GM}\) and \(\mathcal{M}\)
only differs by 
an invertible component-wise transformation of the variables in \(\mathbf{U}\),
therefore \(\mathcal{GM}\) can perform
causal inference tasks just like \(\mathcal{M}\).

\subsection{Topological Batching}

\textit{Topological batching} results in an ordered sequence
\(\mathbf{B} = (\mathbf{B}_1, \cdots, \mathbf{B}_n)\) 
which partitions 
the label set \(\{1, \ldots, d\}\) of a DAG 
\(\mathcal{G} = (\mathbf{V}, \mathbf{E})\).
It provides a method to process \(\mathbf{V}\) sequentially
with a \emph{deterministic} ordering.
The algorithm of topological batching is outlined in the extended version.
In brief, nodes are organized by a topological sort, with each 
\(\mathbf{B}_i\) representing a topological equivalence class of 
\(\mathbf{V}\). 

Topological batching was initially introduced in previous work{\small~\cite{toplogic}} and refined with rigorous mathematical proof by Thost{\small~\cite{dagnn}}.
Since we assume a causal graph is a DAG, topological batching can be directly applied to SCMs. 
See Figure~\ref{fig:topo} for an illustration. 
In this example,
the label set is partitioned into an ordered sequence 
\(\mathbf{B}=(\{1, 2\}, \{3\}, \{4\})\).

\section{Causally Consistent Normalizing Flows}
\label{s:design}

\subsection{Definitions}

\PP{Sequential representation of an SCM}
Similar to the graph representation, 
\textit{The sequential representation of an SCM} 
entails describing the SCM by an ordered sequence.
Specifically, for a causal graph \(\mathcal{G}\) which is a DAG,
we can obtain an ordered sequence 
\(\mathbf{B} = (\mathbf{B}_1, \cdots, \mathbf{B}_n)\) 
by applying topological batching on 
$\mathcal{G}$ as discussed in Section~\ref{s:background}.
The sequence \(\mathbf{B}\) could be interpreted as
a sequential representation of the SCM.
For instance,
the sequential representation of Figure~\ref{fig:topo}
is \((\{1, 2\}, \{3\}, \{4\})\).

\PP{Partial causal transformation}
Assume we have two random variables: \(\mathbf{U} = \{U_1, \cdots, U_d\}\), 
\(\mathbf{X} = \{X_1, \cdots, X_d\}\), 
a label subset \(\mathbf{L} \subseteq \{1, \cdots, d\}\) and 
the parent node label set \({\mathbf{pa}_i}\) of each \(X_i\) 
as defined in Section~\ref{s:background}.
A \textit{partial causal transformation} \(T_\theta\) 
over the label set \(\mathbf{L}\) can be expressed as follows:
\begin{equation}
\label{eq:pcaf}
 X_i = 
\begin{cases}
    T_\theta(U_i \mid \mathbf{U}_{\mathbf{pa}_i}) & \forall i \in \mathbf{L} \\
    U_i & \forall i \notin \mathbf{L} 
\end{cases}    
\end{equation}
We call it "partial" because it only transfers \(U_i\) to \(X_i\) 
under the condition \(\mathbf{U}_{\mathbf{pa}_i}\) 
for any \(i \in \mathbf{L}\).
In the following section, we use \(T_{\theta}^\mathbf{L}\) 
to express the partial causal transformation 
\(T_\theta\) over the label set \(\mathbf{L}\).

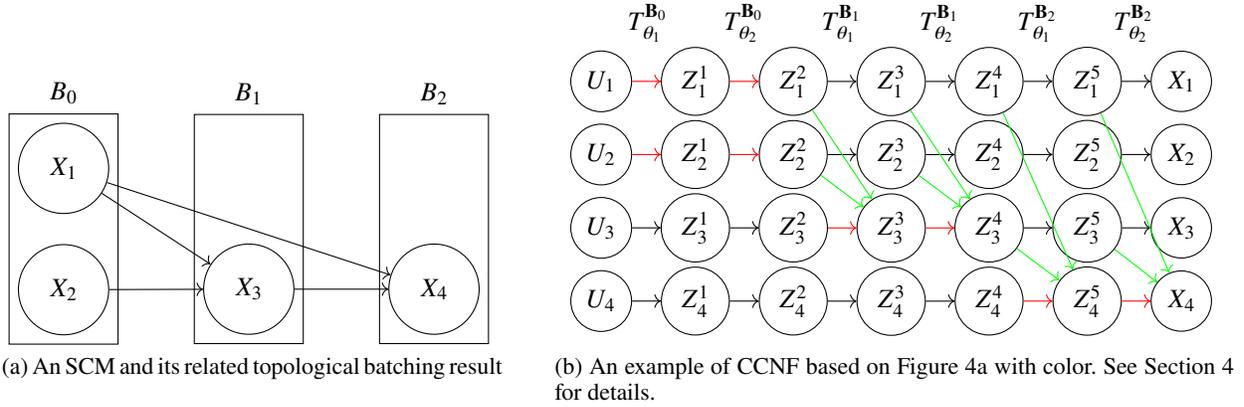
\begin{figure*}[ht]
    \centering
    \subfloat[An SCM and its related topological batching result] {
        \begin{tikzpicture} [
    node distance=1cm, 
    every matrix/.style = {
        matrix of nodes,
        draw, 
        nodes={
            circle,
            minimum size = 1.2cm
        }, 
        row sep=0.4cm
    }
]
    \matrix (m1) {
        |[draw]| $X_1$ \\
        |[draw]| $X_2$ \\
    };
    \node[above] at (m1.north) {$B_0$};
    
    \matrix (m2) [right = of m1]{
        ~ \\
       |[draw]| $X_3$ \\
    };
    \node[above] at (m2.north) {$B_1$};
     
    \matrix (m3) [right = of m2]{
        ~ \\
        |[draw]| $X_4$ \\
    };
    \node[above] at (m3.north) {$B_2$};

    \draw[->] (m1-1-1) -- (m2-2-1);
    \draw[->] (m1-2-1) -- (m2-2-1);
    \draw[->] (m2-2-1) -- (m3-2-1);
    \draw[->] (m1-1-1) -- (m3-2-1);
\end{tikzpicture}
        \label{fig:topo}
    }
    \hfil
    \subfloat[An example of \sys based on Figure~\ref{fig:topo}
    with color.
    See Section~\ref{s:design} for details.] {
        \begin{tikzpicture} [
    node distance=0.2cm, 
    every matrix/.style = {
        matrix of nodes,
        nodes in empty cells,
        nodes={
            circle,
            minimum size = 0.2cm
        }, 
        row sep=0.05cm,
        column sep=0.38cm
    }
]
    \matrix (m) {
        |[draw]| $U_1$ & |[draw]| $Z^1_1$ & |[draw]| $Z^2_1$ & |[draw]| $Z^3_1$ & |[draw]| $Z^4_1$ & |[draw]| $Z^5_1$ & |[draw]| $X_1$ \\  
        |[draw]| $U_2$ & |[draw]| $Z^1_2$ & |[draw]| $Z^2_2$ & |[draw]| $Z^3_2$ & |[draw]| $Z^4_2$ & |[draw]| $Z^5_2$  & |[draw]| $X_2$ \\
        |[draw]| $U_3$ & |[draw]| $Z^1_3$ & |[draw]| $Z^2_3$ & |[draw]| $Z^3_3$ & |[draw]| $Z^4_3$ & |[draw]| $Z^5_3$  & |[draw]| $X_3$ \\
        |[draw]| $U_4$ & |[draw]| $Z^1_4$ & |[draw]| $Z^2_4$ & |[draw]| $Z^3_4$ & |[draw]| $Z^4_4$ & |[draw]| $Z^5_4$  & |[draw]| $X_4$ \\
    };
  \coordinate (midpoint1) at ($(m-1-1)!0.5!(m-1-2)$);
  \node[above=0.4cm of midpoint1] {$T_{\theta_1}^{\mathbf{B}_0}$};

  \coordinate (midpoint2) at ($(m-1-2)!0.5!(m-1-3)$);
  \node[above=0.4cm of midpoint2] {$T_{\theta_2}^{\mathbf{B}_0}$};

  \coordinate (midpoint3) at ($(m-1-3)!0.5!(m-1-4)$);
  \node[above=0.4cm of midpoint3] {$T_{\theta_1}^{\mathbf{B}_1}$};

  \coordinate (midpoint4) at ($(m-1-4)!0.5!(m-1-5)$);
  \node[above=0.4cm of midpoint4] {$T_{\theta_2}^{\mathbf{B}_1}$};

  \coordinate (midpoint5) at ($(m-1-5)!0.5!(m-1-6)$);
  \node[above=0.4cm of midpoint5] {$T_{\theta_1}^{\mathbf{B}_2}$};

  \coordinate (midpoint6) at ($(m-1-6)!0.5!(m-1-7)$);
  \node[above=0.4cm of midpoint6] {$T_{\theta_2}^{\mathbf{B}_2}$};
    

  \draw[red, ->] (m-1-1) -- (m-1-2);
  \draw[red, ->] (m-1-2) -- (m-1-3);
  \draw[->] (m-1-3) -- (m-1-4);
  \draw[->] (m-1-4) -- (m-1-5);
  \draw[->] (m-1-5) -- (m-1-6);
  \draw[->] (m-1-6) -- (m-1-7);

  \draw[red,->] (m-2-1) -- (m-2-2);
  \draw[red,->] (m-2-2) -- (m-2-3);
  \draw[->] (m-2-3) -- (m-2-4);
  \draw[->] (m-2-4) -- (m-2-5);
  \draw[->] (m-2-5) -- (m-2-6);
  \draw[->] (m-2-6) -- (m-2-7);

  \draw[->] (m-3-1) -- (m-3-2);
  \draw[->] (m-3-2) -- (m-3-3);
  \draw[red, ->] (m-3-3) -- (m-3-4);
  \draw[red, ->] (m-3-4) -- (m-3-5);
  \draw[->] (m-3-5) -- (m-3-6);
  \draw[->] (m-3-6) -- (m-3-7);

  \draw[->] (m-4-1) -- (m-4-2);
  \draw[->] (m-4-2) -- (m-4-3);
  \draw[->] (m-4-3) -- (m-4-4);
  \draw[->] (m-4-4) -- (m-4-5);
  \draw[red, ->] (m-4-5) -- (m-4-6);
  \draw[red, ->] (m-4-6) -- (m-4-7);

  \draw[green, ->] (m-1-3) -- (m-3-4);
  \draw[green, ->] (m-1-4) -- (m-3-5);
  \draw[green, ->] (m-2-3) -- (m-3-4);
  \draw[green, ->] (m-2-4) -- (m-3-5);

  \draw[green, ->] (m-1-5) -- (m-4-6);
  \draw[green, ->] (m-1-6) -- (m-4-7);
  \draw[green, ->] (m-3-5) -- (m-4-6);
  \draw[green, ->] (m-3-6) -- (m-4-7);
\end{tikzpicture}
        \label{fig:example}
    }
    \caption{An example SCM, its topological order, and a related \sys.}
\end{figure*}

\PP{Causally Consistent Normalizing Flows}
Given the sequential representation \(\mathbf{B} = (\mathbf{B}_1, \cdots, \mathbf{B}_n)\) of an SCM \(\mathcal{M}\),
we define \textit{Causally Consistent Normalizing Flows} as such NF whose transformation is \(T_{\theta}^{\mathbf{B}} = T_{\theta_n}^{\mathbf{B}_n} \circ \cdots \circ T_{\theta_1}^{\mathbf{B}_1}\).

For convenience, we use \(\mathbf{Z}\) = \((\mathbf{Z}^0 = \mathbf{U}, \cdots, \mathbf{Z}^{n - 1}, \mathbf{Z}^{n} = \mathbf{X})\) to represent the output of each NF in the chain.
More specifically, we use \(Z^{k}_{i}\) to represent the \(i\)-th output of \(\mathbf{Z}^{k}\).

\subsection{A Running Example}
We implement \sys based on the SCM
illustrated in Figure~\ref{fig:topo} as an example.
Recall that the sequential representation of Figure~\ref{fig:topo} is \(\mathbf{B} = (\mathbf{B}_0, \mathbf{B}_1, \mathbf{B}_2) = (\{1, 2\}, \{3\}, \{4\})\).
The \sys comprises two layers for each \(T_{\theta_i}^{\mathbf{B}_i}\) and is depicted in Figure~\ref{fig:example}.
Note that there is no limitation on the number of layers for \(T_{\theta_i}^{\mathbf{B}_i}\),
and we pick two for convenience.

The details of the graph in Figure~\ref{fig:example}
are elaborated below.
Each column represents \((\mathbf{Z}^0, \cdots, \mathbf{Z}^{n})\) as previously described.
We use right arrows with different colors to denote the direction of the dataflow.
The black arrow indicates the left node is equivalent to the right node, such as \(U_3 \to Z_3^1\) meaning \(U_3 == Z_3^1\).
The green arrow indicates that the value of the left node is utilized during training to determine the parameters.
For instance, the green arrows between \(Z_1^2 {\to} Z_3^3\) and \(Z_2^2 {\to} Z_3^3\) suggest that \(Z_1^2\) and \(Z_2^2\) are employed to determine the parameter \(\theta_1\).
The red arrow indicates the left variable is the dependent variable on the right variable.
For instance, \(Z_3^2 {\to} Z_3^3\) denotes \(Z_3^3 = T_{\theta_1}^\mathbf{B_1}(Z_3^3)\).
Specifically,
since there is no dependence for the variables $X_1, X_2 \in \mathbf{B}_0$,
the \(T_{\theta_0}^{\mathbf{B}_0}\) could be emitted
and we can obtain the value of $X_1, X_2$
by directly sampling from the distribution.

\section{\sys: Properties and Operations}
\label{s:operations}

%

\subsection{Properties of \sys}
\begin{theorem}[Causality]
    \label{causality}
    Given a \sys \(T_{\theta}^\mathbf{B}\),
    for the \(i\)-th variable \(X_i\),
    \(X_i\) only depends on its parents \(\mathbf{X}_{\mathbf{pa}_i}\) and \(U_i\).
    Particularly assume \(i \in \mathbf{B}_j\), we have \(X_i = T_{\theta_j}^\mathbf{B_j}(U_i \mid \mathbf{X}_{\mathbf{pa}_i})\)
\end{theorem}
\begin{theorem}[Universality]
    \label{universality}
    A \sys \(T_{\theta}^\mathbf{B}\) is a multi-layer universal approximator as long as for any \(j\), \(T_{\theta_j}^{\mathbf{B}_j}\) is a multi-layer universal approximator.
\end{theorem}
\begin{theorem}[Causal Consistency]
    \label{core}
    \(T_{\theta}^\mathbf{B}\) is causally consistent with the given SCM \(\mathcal{M}\).
\end{theorem}
\begin{theorem}[Minimum Layer]
    \label{mini}
    If the longest path of the DAG causal graph \(\mathcal{G}\) is \(d\), then \sys contains at least \(d\) layers. 
\end{theorem}

Proofs of the theorems
are available in the extended version.
All those properties together make \sys highly practical.
Causality guarantees that all causal relationships within the SCM are encapsulated.
Universality guarantees \sys can approximate the distributions of endogenous random variables in any form.
Causal consistency and minimum layer limitation collectively guarantee \sys can capture accurate causality in \(\mathcal{M}\).

\subsection{Causal Inference Tasks}
According to Pearl’s causal hierarchy, causal inference tasks can be divided into three levels: observations, interventions, and counterfactuals.
Here we will demonstrate that \sys can perform all three tasks effectively.
The extended version contains algorithms
for all three tasks.

\PP{Observations}
Generating observation data in \sys is straightforward.
We begin by sampling \(\mathbf{U}\) from the given distributions \(P_{\mathbf{U}}\).
Then we compute \(\mathbf{X}\) through \(\mathbf{X} = T_{\theta}^\mathbf{B}(\mathbf{U})\).
We repeat this process and
get multiple possible \(\mathbf{X}\)s 
to form the observations $\mathbf{O}$.

\PP{The Do Operator}
Before diving into interventions and counterfactuals, it is crucial to introduce the \textit{do operator}{\small~\cite{doop}},
since it forms the foundation for those concepts.
\(Do(X_i = a)\) simulates a physical intervention on an SCM \(\mathcal{M} = (\widetilde{\mathbf{f}}, P_\mathrm{u})\) by fixing the observational variable \(X_i\) to a specific value \(a\).
Traditionally, the do operator requires modifying the \(\mathcal{M}\) by substituting the variable \(X_i\) from \(a\) in every function \(\widetilde{\mathbf{f}}\).
However, this approach is not suitable for \sys.
Instead, we propose a method like Causal NF to address this limitation.
The key insight lies on the fact that \(X_i = T_{\theta_j}^\mathbf{B_j}(U_i \mid \mathbf{X}_{\mathbf{pa}_i})\), indicating that fixing the value of \(X_i\) is equivalent to fixing the value of \(U_i\).

\PP{Interventions}
Interventions can be realized as applications of the do operator.
More particularly, given \(\mathbf{X} = (X_1, \cdots, X_d)\) with \(d\) attributes,
interventions inquire about the distributions of variables when \(X_i\) is fixed to \(a\), i.e. \(P(X_j \mid X_i = a), j \in \{1, \ldots, d\}\).
In practice, 
we generate observations \(\mathbf{O}\)
as described before.
For each \(\mathbf{X} \in \mathbf{O}\),
we constraint the value of $X_i$ in \(\mathbf{X}\)
by performing \(Do(X_i = a)\) on $\mathbf{X}$.
The constrained results represent samples from the distributions \(P(\mathbf{X} \mid Do(X_i = a))\).

\PP{Counterfactuals}
Like interventions, counterfactuals can also be accomplished through the do operator.
Specifically,
counterfactuals seek the precise value of \(X_j, j \in \{1, \ldots, d\}\) when the set \(X_i\) is fixed to its counterfactual, i.e. \(X_i \gets X_{i}^{cf}\).
In practice,
For any given $\mathbf{X}$,
we execute \(Do(X_i = X_{i}^{cf})\) on it 
to get the counterfactual of \(\mathbf{X}\).
\section{Evaluation}
\label{s:eval}

\noindent
We evaluate \sys to answer three key questions:
\begin{enumerate}[align=left, leftmargin=*]
\item \PP{Causal Consistency}
Despite theoretical assurances of causal consistency
is demonstrated,
does \sys maintain this consistency 
in practical implementations?

\item \PP{Performance on Causal Inference Tasks}
In causal inference tasks,
How accurately do the data instances generated by \sys 
compare with those generated by actual models? 
Is there an observable improvement in accuracy compared to state-of-the-art models?

\item \PP{Effectiveness in Real-world Case Studies}
Can \sys be effectively applied to real-world scenarios, such as mitigating unfairness?
\end{enumerate}
\noindent
Refer to the links of extended version and code 
for a complete description of the experiments.

\subsection{Causal Consistency}

\begin{table*}[ht]
    \centering
    {\small
    \begin{tabularx}{\textwidth}{XXXXXXXX}
\toprule
    &
    \multicolumn{3}{c}{L = 1} &
    \multicolumn{4}{c}{L $>$ 1}
    \\  

\cmidrule(r){2-4}
\cmidrule(r){5-8}

    &
    CAREFL &
    \textbf{CausalNF} &
    \textbf{VACA} &
    CAREFL &
    CausalNF &
    VACA &
    \textbf{\sys} \\

\midrule
    $KL$ &
    0.01\(\pm\)0.03 &
    0.00\(\pm\)0.00 &
    2.96\(\pm\)0.08 &
    0.00\(\pm\)0.00 &
    0.00\(\pm\)0.00 &
    2.62\(\pm\)0.08 &
    0.00\(\pm\)0.00\\
    
    $\mathcal{L}(T_\theta(X))$ &
    0.20\(\pm\)0.04 &
    \textbf{0.00}\(\pm\)\textbf{0.00} &
    \textbf{0.00}\(\pm\)\textbf{0.00} &
    0.32\(\pm\)0.09 &
    0.16\(\pm\)0.05 &
    0.15\(\pm\)0.01 &
    \textbf{0.00}\(\pm\)\textbf{0.00}\\
\bottomrule 
\end{tabularx}

}
    \caption{Causal consistency comparison between \sys and prior works. The causally consistent models are marked in bold. KL is used to evaluate the accuracy and $\mathcal{L}(T_\theta(X))$ is used to evaluate the causal consistency}
    \label{tab:cc}
\end{table*}

\PP{Experiment design}
We compare \sys with three state-of-the-art models: CAREFL, VACA and CausalNF.
CAREFL is the first causal autoregressive flow, utilizing causal ordering with an affine layer to capture casualty.
VACA employs a GNN to encode the causal graph, leveraging its structure to capture causal relationships.
CausalNF restricts the conditioner of autoregressive flow to the parent nodes, enhancing its ability to model causal dependencies.
All models except \sys are categorized into two types: a model with one layer ($L = 1$) and a model with more than one layer ($L > 1$).
More details are in the extended version.

\PP{Measurement}
In this experiment, we evaluate given models based on two key metrics: accuracy and causal consistency.
Accuracy reflects through the KL distance between the captured prior distribution and the actual one, denoted as $KL(p_\mathcal{M} \vert p_\theta)$.
Causal consistency is reflected through 
Equation~\ref{eq:cceq} as in Causal NF. 
\begin{equation}
    \label{eq:cceq}
    \mathcal{L}(T_\theta(X)) = \Vert \nabla_xT_\theta(X) \cdot (1 - G) \Vert_2
\end{equation}
Here, \(G\) represents 
the causal graph as an adjacency matrix 
of a given SCM $\mathbf{M}$, 
\(\nabla_xT_\theta(X)\) denotes the Jacobian matrix of \(T_\theta(X)\).
\(T_\theta(X)\) is causally consistent with $\mathbf{M}$
iff \(\mathcal{L}(T_\theta(X)) = 0\).

\PP{Result}
Results are summarized in Table~\ref{tab:cc}.
In a nutshell, the practical results are consistent with theoretical expectations.
Firstly, \sys demonstrates causally consistent with the given SCM, as \(\mathcal{L}(T_\theta(X))\) of \sys is consistently 0.
Secondly, state-of-the-art models can only keep consistency by constraining their expressive power.
CAREFL fails to meet the casual consistency altogether,
while VACA and CauslNF can only achieve that with a single layer ($L = 1$).
These findings remain consistent with the results reported in their respective papers.

\subsection{Causal Inference Tasks}

\begin{table*}[!t]
    \centering
    {\small
    
\begin{tabularx}{\textwidth}{XXXXXXX}

\toprule
\multirow{2}{*}{Dataset} &
\multirow{2}{*}{Model} &
\multicolumn{3}{c}{Performance} &
\multicolumn{2}{c}{Time($ms$)} \\

\cmidrule(r){3-5}
\cmidrule(r){6-7}

&
&
KL &
Inter.\textsubscript{MMD} &
C.F.\textsubscript{RMSD} &
Train &
Evaluation \\
\midrule

\multirow{3}{*}{\shortstack{Nlin-\\Triangle}} &
CCNF &
\textbf{0.12}$\pm$\textbf{0.04} &
\textbf{0.03}$\pm$\textbf{0.03} &
\textbf{0.14}$\pm$\textbf{0.05} &
20.24$\pm$0.32 &
19.80$\pm$0.27 \\

&
CausalNF &
0.37$\pm$0.00 &
0.10$\pm$0.02 &
0.79$\pm$0.07 &
\textbf{6.03}$\pm$\textbf{0.07} &
\textbf{5.09}$\pm$\textbf{0.10}\\

&
VACA &
1.41$\pm$0.07 &
2.13$\pm$0.61 &
34.62$\pm$12.53 &
36.23$\pm$0.70 &
35.27$\pm$0.63 \\
\hline

\multirow{3}{*}{\shortstack{Nlin-\\Simpson}} &
CCNF &
\textbf{0.01}$\pm$\textbf{0.00} &
\textbf{0.00}$\pm$\textbf{0.00} &
\textbf{0.00}$\pm$\textbf{0.00} &
15.96$\pm$0.24 &
15.73$\pm$0.68 \\

&
CausalNF &
0.25$\pm$0.00 &
0.04$\pm$0.01 &
\textbf{0.00}$\pm$\textbf{0.00} &
\textbf{6.30}$\pm$\textbf{0.13} &
\textbf{5.53}$\pm$\textbf{0.30} \\

&
VACA &
1.56$\pm$0.04 &
0.11$\pm$0.15 &
0.59$\pm$0.10 &
36.52$\pm$0.40 &
35.62$\pm$0.39 \\
\hline

\multirow{3}{*}{M-Graph} &
CCNF &
\textbf{0.01}$\pm$\textbf{0.00} &
\textbf{0.00}$\pm$\textbf{0.00} &
0.04$\pm$0.01 &
9.13$\pm$0.11 &
8.36$\pm$0.18 \\

&
CausalNF &
0.32$\pm$0.00 &
0.17$\pm$0.01 &
\textbf{0.02}$\pm$\textbf{0.01} &
\textbf{6.21}$\pm$\textbf{0.15} &
\textbf{5.33}$\pm$\textbf{0.26} \\

&
VACA &
1.83$\pm$0.01 &
0.12$\pm$0.01 &
1.19$\pm$0.03 &
40.50$\pm$1.18 &
39.89$\pm$0.72 \\
\hline

\multirow{3}{*}{Network} &
CCNF &
\textbf{0.55}$\pm$\textbf{0.13} &
\textbf{0.01}$\pm$\textbf{0.01} &
\textbf{0.25}$\pm$\textbf{0.06} &
19.94$\pm$0.32 &
19.44$\pm$0.86 \\

&
CausalNF &
1.38$\pm$0.04 &
0.15$\pm$0.02 &
0.41$\pm$0.03 &
\textbf{9.11}$\pm$\textbf{0.08} &
\textbf{8.14}$\pm$\textbf{0.07}\\

&
VACA &
1.67$\pm$0.03 &
0.38$\pm$0.12 &
13.20$\pm$0.50 &
38.36$\pm$0.36 &
37.59$\pm$0.36 \\
\hline

\multirow{3}{*}{Backdoor} &
CCNF &
\textbf{0.56}$\pm$\textbf{0.00} &
\textbf{0.01}$\pm$\textbf{0.00} &
0.04$\pm$0.01 &
15.21$\pm$0.22 &
14.71$\pm$0.62 \\

&
CausalNF &
0.96$\pm$0.00 &
0.08$\pm$0.02 &
\textbf{0.04}$\pm$\textbf{0.00} &
\textbf{6.30}$\pm$\textbf{0.23} &
\textbf{5.45}$\pm$\textbf{0.40} \\

&
VACA &
1.78$\pm$0.01 &
0.13$\pm$0.01 &
1.88$\pm$0.02 &
39.04$\pm$0.39 &
38.21$\pm$0.40 \\
\hline

\multirow{3}{*}{Chain} &
CCNF &
\textbf{0.28}$\pm$\textbf{0.02} &
\textbf{0.00}$\pm$\textbf{0.00} &
\textbf{0.03}$\pm$\textbf{0.00} &
26.10$\pm$0.34 &
24.78$\pm$0.29 \\

&
CausalNF &
4.67$\pm$0.04 &
0.05$\pm$0.01 &
0.13$\pm$0.02 &
\textbf{11.59}$\pm$\textbf{0.09} &
\textbf{10.75}$\pm$\textbf{0.07} \\

&
VACA &
1.52$\pm$0.14 &
0.22$\pm$0.03 &
1.26$\pm$0.26 &
45.69$\pm$0.72 &
44.66$\pm$0.66 \\

\bottomrule
\end{tabularx}
}
    \caption{Causal inference tasks comparison between causally consistent models. In the title, \textit{Inter.} means interventions, \textit{C.F.} means counterfactuals. The best results are marked in bold}
    \label{tab:comparison}
\end{table*}

\PP{Experiment design}
Since only CausalNF and VACA with one layer ($L$ = 1) can ensure causal consistency, 
in this experiment we compare \sys with
CausalNF ($L$ = 1) and VACA ($L$ = 1)
for causal inference tasks.
We test given models on representative synthetic datasets: Nonlinear Triangle dataset, Nonlinear Simpson dataset, M-graph dataset, Network dataset, Backdoor dataset, and Chain dataset with 3–8 nodes, respectively.
Those causal structures are either from previous works{\small~\cite{causalnf, vaca}}
or from practical applications,
making them suitable for evaluating the performance of causal inference.

\PP{Measurement}
We use three different measurements to evaluate the performance of causal inference tasks.
For observations, the measurement is the KL distance,
for interventions, we measure the max \textit{Maximum Mean Discrepancy} (MMD) distance.
For counterfactuals, we measure the 
\textit{Root-Mean-Square Deviation} (RMSD) distance.
More details are in the extended version.

\PP{Result}
As shown in Table~\ref{tab:comparison},
\sys demonstrates superior performance compared
with previous works across nearly all datasets.
This can be attributed to the ability of \sys to capture complex casualties with additional middle layers---an ability absent in
CausalNF and VACA.
While compared with CausalNF, 
\sys requires approximately 130\% more time
for training and evaluation,
the time spent remains within a reasonable range and is less than that of VACA.
Overall, \sys is a more practical choice for causal inference tasks compared to stat-of-the-art tools.

\subsection{Real-world Evaluation}

\begin{table}[t!]
    \centering
    {\small
    \begin{tabular}{llll}
\toprule

Name &
Accuracy &
F1 &
Fairness \\
\midrule

SVM &
73.00$\pm$0.00 &
82.12$\pm$0.00 &
9.00$\pm$0.00 \\

SVM\textsubscript{CF} &
72.60$\pm$1.10 &
81.90$\pm$0.59 &
4.10$\pm$1.40 \\

CCNF &
\textbf{75.80}$\pm$\textbf{2.22} &
\textbf{84.34}$\pm$\textbf{2.22} &
\textbf{0.00}$\pm$\textbf{0.00} \\
\bottomrule
\end{tabular}}
    \caption{Real-world evaluation on German credit dataset, every number is magnified 100 times.}
    \label{tab:german}
\end{table}

\PP{Experiment design}
Like prior works{\small~\cite{vaca, causalnf}}, we select the German credit dataset as a representative example.
Classifiers commonly utilize this dataset to predict the credit risk for a given applicant.
If a classifier of the German credit predicts the risk directly through the \textit{sex} attribute, we say it has the \textit{unfairness problem}.

\sys offers two methods to address real-world issues:
\WC{1} building a fairness classifier directly or strengthening the dataset through counterfactual data augmentation.
Specifically, we first train \sys on the German credit dataset.
For any applicant, \sys could function as an unfairness-free classifier by setting the exogenous variable of the risk attribute to its mean value, which is 0 in our experiment.   
Additionally,
\WC{2} \sys could generate counterfactuals of the training set to create a data-augmented dataset.
We also build an SVM classifier on the origin German credit dataset and the augmented one respectively for comparison.

\PP{Measurement}
We utilize the \textit{individual fairness}{\small~\cite{ifairness}} to evaluate the fairness of a classifier.
Individual fairness is determined by examining whether changing the sex attribute of an applicant alters the risk level predicted by a classifier.
Mathematically, for a test dataset with $n$ instances, we define \textit{individual fairness} as $\sum_{i = 1}^{n} |Risk(X_{sex = 1}) - Risk(X_{sex = 0})| / n$.

\PP{Result}
The results are summarized in Table~\ref{tab:german}.
Overall, \sys can enhance fairness while maintaining accuracy in both methods.
For the fairness problem, \sys can lower the fairness rate significantly.
Particularly, the NF classifier demonstrates superior accuracy and eliminates unfairness.
Those facts reveal that \sys are suitable for real-world problems like unfairness without compromising accuracy.
\section{Conclusion and Future Work}
\label{s:conclusion}
Causal inconsistency in generative models can lead to significant consequences.
While some prior works cannot guarantee causal consistency, others achieve it only by limiting their depth.
In the paper, we introduce \sys, a novel causal GM that ensures causal consistency without limiting the depth as rigorously demonstrated in Section~\ref{s:operations}.
Furthermore, we elaborate on how to perform causal inference tasks in \sys, 
demonstrating its proficiency in efficiency.
Through synthetic experiments, we illustrate that \sys outperforms state-of-the-art models 
in terms of accuracy 
across different causal tasks.
To validate the real-world applicability of \sys, 
we also apply \sys to a real-world dataset and 
address practical issues concerning unfairness.
Our results indicate that \sys can effectively 
mitigate problems while maintaining accuracy and reliability.
Overall, \sys represents a significant advancement in incorporating generative models with causality,
offering both theoretical guarantees of causal consistency and practical applicability
in addressing real-world issues involved with causality.

\PP{Future work} 
First, 
a major constraint of \sys is the requirement of a well-defined causal graph, which is challenging to obtain in reality.
In the future, \sys should be able to handle these cases 
more properly by supporting even incorrect or non-DAG causal graphs.
Second,
while \sys primarily focuses on causal inference tasks, 
its capabilities could be extended to other domains.
For instance,
by incorporating appropriate causalities to the SCM, 
\sys could aid in causal discovery tasks.
In the future, \sys could be leveraged for extensive tasks.

\bibliography{reference}

\begin{thebibliography}{25}
\providecommand{\natexlab}[1]{#1}

\bibitem[{Agarap(2018)}]{relu}
Agarap, A.~F. 2018.
\newblock Deep learning using rectified linear units (relu).
\newblock \emph{arXiv preprint arXiv:1803.08375}.

\bibitem[{Akiba et~al.(2019)Akiba, Sano, Yanase, Ohta, and Koyama}]{optuna}
Akiba, T.; Sano, S.; Yanase, T.; Ohta, T.; and Koyama, M. 2019.
\newblock Optuna: A next-generation hyperparameter optimization framework.
\newblock In \emph{Proceedings of the 25th ACM SIGKDD international conference on knowledge discovery \& data mining}, 2623--2631.

\bibitem[{Crouse et~al.(2019)Crouse, Abdelaziz, Cornelio, Thost, Wu, Forbus, and Fokoue}]{toplogic}
Crouse, M.; Abdelaziz, I.; Cornelio, C.; Thost, V.; Wu, L.; Forbus, K.; and Fokoue, A. 2019.
\newblock Improving graph neural network representations of logical formulae with subgraph pooling.
\newblock \emph{arXiv preprint arXiv:1911.06904}.

\bibitem[{Draxler et~al.(2024)Draxler, Wahl, Schn{\"o}rr, and K{\"o}the}]{cnfuniversality}
Draxler, F.; Wahl, S.; Schn{\"o}rr, C.; and K{\"o}the, U. 2024.
\newblock On the universality of coupling-based normalizing flows.
\newblock \emph{arXiv preprint arXiv:2402.06578}.

\bibitem[{Fleisher(2021)}]{ifairness}
Fleisher, W. 2021.
\newblock What's fair about individual fairness?
\newblock In \emph{Proceedings of the 2021 AAAI/ACM Conference on AI, Ethics, and Society}, 480--490.

\bibitem[{Harshvardhan et~al.(2020)Harshvardhan, Gourisaria, Pandey, and Rautaray}]{gmsurvey}
Harshvardhan, G.; Gourisaria, M.~K.; Pandey, M.; and Rautaray, S.~S. 2020.
\newblock A comprehensive survey and analysis of generative models in machine learning.
\newblock \emph{Computer Science Review}, 38: 100285.

\bibitem[{Hofmann(1994)}]{german}
Hofmann, H. 1994.
\newblock {Statlog (German Credit Data)}.
\newblock UCI Machine Learning Repository.
\newblock {DOI}: https://doi.org/10.24432/C5NC77.

\bibitem[{Javaloy, Martin, and Valera(2023)}]{causalnf}
Javaloy, A.; Martin, P.~S.; and Valera, I. 2023.
\newblock Causal normalizing flows: from theory to practice.
\newblock In \emph{Thirty-seventh Conference on Neural Information Processing Systems}.

\bibitem[{Khemakhem et~al.(2021)Khemakhem, Monti, Leech, and Hyvärinen}]{carefl}
Khemakhem, I.; Monti, R.~P.; Leech, R.; and Hyvärinen, A. 2021.
\newblock Causal Autoregressive Flows.
\newblock arXiv:2011.02268.

\bibitem[{Kingma and Ba(2014)}]{adam}
Kingma, D.~P.; and Ba, J. 2014.
\newblock Adam: A method for stochastic optimization.
\newblock \emph{arXiv preprint arXiv:1412.6980}.

\bibitem[{Kocaoglu et~al.(2017)Kocaoglu, Snyder, Dimakis, and Vishwanath}]{causalgan}
Kocaoglu, M.; Snyder, C.; Dimakis, A.~G.; and Vishwanath, S. 2017.
\newblock CausalGAN: Learning Causal Implicit Generative Models with Adversarial Training.
\newblock arXiv:1709.02023.

\bibitem[{Komanduri et~al.(2024)Komanduri, Wu, Wu, and Chen}]{vaefuture}
Komanduri, A.; Wu, X.; Wu, Y.; and Chen, F. 2024.
\newblock From Identifiable Causal Representations to Controllable Counterfactual Generation: A Survey on Causal Generative Modeling.
\newblock arXiv:2310.11011.

\bibitem[{Papamakarios et~al.(2021)Papamakarios, Nalisnick, Rezende, Mohamed, and Lakshminarayanan}]{nf}
Papamakarios, G.; Nalisnick, E.; Rezende, D.~J.; Mohamed, S.; and Lakshminarayanan, B. 2021.
\newblock Normalizing flows for probabilistic modeling and inference.
\newblock \emph{Journal of Machine Learning Research}, 22(57): 1--64.

\bibitem[{Papamakarios, Pavlakou, and Murray(2018)}]{maf}
Papamakarios, G.; Pavlakou, T.; and Murray, I. 2018.
\newblock Masked Autoregressive Flow for Density Estimation.
\newblock arXiv:1705.07057.

\bibitem[{Pawlowski, Coelho~de Castro, and Glocker(2020)}]{deepscm}
Pawlowski, N.; Coelho~de Castro, D.; and Glocker, B. 2020.
\newblock Deep structural causal models for tractable counterfactual inference.
\newblock \emph{Advances in neural information processing systems}, 33: 857--869.

\bibitem[{Pearl(2009)}]{pearl}
Pearl, J. 2009.
\newblock \emph{Causality}.
\newblock Cambridge university press.

\bibitem[{Pearl(2012)}]{doop}
Pearl, J. 2012.
\newblock The do-calculus revisited.
\newblock \emph{arXiv preprint arXiv:1210.4852}.

\bibitem[{Ribeiro et~al.(2023)Ribeiro, Xia, Monteiro, Pawlowski, and Glocker}]{imagepcm}
Ribeiro, F. D.~S.; Xia, T.; Monteiro, M.; Pawlowski, N.; and Glocker, B. 2023.
\newblock High Fidelity Image Counterfactuals with Probabilistic Causal Models.
\newblock arXiv:2306.15764.

\bibitem[{Sanchez and Tsaftaris(2022)}]{difscm}
Sanchez, P.; and Tsaftaris, S.~A. 2022.
\newblock Diffusion causal models for counterfactual estimation.
\newblock \emph{arXiv preprint arXiv:2202.10166}.

\bibitem[{S{\'a}nchez-Martin, Rateike, and Valera(2022)}]{vaca}
S{\'a}nchez-Martin, P.; Rateike, M.; and Valera, I. 2022.
\newblock VACA: Designing variational graph autoencoders for causal queries.
\newblock In \emph{Proceedings of the AAAI Conference on Artificial Intelligence}, volume~36, 8159--8168.

\bibitem[{Thompso(2023)}]{facebook}
Thompso, E. 2023.
\newblock Class-action lawsuit against Facebook claiming discrimination gets the green light.
\newblock \emph{CBC}.

\bibitem[{Thost and Chen(2021)}]{dagnn}
Thost, V.; and Chen, J. 2021.
\newblock Directed Acyclic Graph Neural Networks.
\newblock arXiv:2101.07965.

\bibitem[{Xi and Bloem-Reddy(2023)}]{indeterminacy}
Xi, Q.; and Bloem-Reddy, B. 2023.
\newblock Indeterminacy in generative models: Characterization and strong identifiability.
\newblock In \emph{International Conference on Artificial Intelligence and Statistics}, 6912--6939. PMLR.

\bibitem[{Xia, Pan, and Bareinboim(2022)}]{nerualcm}
Xia, K.; Pan, Y.; and Bareinboim, E. 2022.
\newblock Neural causal models for counterfactual identification and estimation.
\newblock \emph{arXiv preprint arXiv:2210.00035}.

\bibitem[{Yang et~al.(2021)Yang, Liu, Chen, Shen, Hao, and Wang}]{causalvae}
Yang, M.; Liu, F.; Chen, Z.; Shen, X.; Hao, J.; and Wang, J. 2021.
\newblock Causalvae: Disentangled representation learning via neural structural causal models.
\newblock In \emph{Proceedings of the IEEE/CVF conference on computer vision and pattern recognition}, 9593--9602.

\end{thebibliography}
\clearpage
\appendix
\section{Algorithm}
\label{appendix:algo}

\subsection{Topological batching}
\begin{algorithm}[h]
\caption{Topological batching algorithm}
\label{alg:topo}
\KwData{
A DAG \(\mathcal{G} = (\mathbf{V}, \mathbf{E})\);
\(|\mathbf{V}| = d\)
}
\KwResult{
\(\mathbf{B} = (\mathbf{B}_1, \cdots, \mathbf{B}_n)\ is\ a\ partition\ of\ \{1, \cdots, d\}\)
}
\(\mathbf{B} \gets ()\)\; 
\(i \gets 1\)\; 
\While{\(\mathbf{V}\) is not empty} {
    \(\mathbf{B}_i = \{k \mid V_k \in \mathbf{V} \wedge V_k\ has\ no\ predecessor\}\) \; 
    \(\mathbf{B}.append(\mathbf{B}_i)\)\; 
    \(\mathbf{V}.remove(\mathbf{V}_{\mathbf{B}_i})\)\; 
    \(i \gets i + 1\)\; 
}
\end{algorithm}

\subsection{Causal Inference tasks in \sys}
\begin{algorithm}[h]
\caption{Causal Inference tasks in \sys}\label{alg:causaltask}

\SetKwProg{Proc}{Procedure}{:}{End\ procedure}
\Proc{ObservationsSample(n)}{
    \For{i = 1:n} {
    $\mathbf{U}_i \sim P_{\mathbf{U}}(\mathbf{U})$\;
    $\mathbf{X}_i = T_{\theta}^\mathbf{B}(\mathbf{U}_i)$\;
    }
    \Return $\mathbf{O} = \{\mathbf{X}_1 \cdots \mathbf{X}_n\}$\;
}
\BlankLine
\Proc{Do(X, i, a)}{
    $\mathbf{U} = T_{\theta}^{\mathbf{B}, -1}(\mathbf{X})$\;
    $\mathbf{U}_i \gets T_{\theta_j}^{\mathbf{B_j}, -1}(a \mid \mathbf{X}_{\mathbf{pa}_i})$\;
    \Return $\mathbf{S}$\;
}
\BlankLine
\Proc{InterventionsSample(n, i, a)} {
    $\mathbf{O}$ = $ObservationsSample(n)$\;
    $\mathbf{S} = \{\}$\;
    \For {$\mathbf{X}$ in $\mathbf{O}$} {
        $\mathbf{S}.insert(Do(\mathbf{X}, i, a))$
    }
    \Return $\mathbf{S}$\;
}
\BlankLine
\Proc{Couterfactuals(X, i)}{
    \Return \(Do(\mathbf{X}, i, X_{i}^{cf})\)
}
\end{algorithm}
\section{Proof of theories}
\label{appendix:proof}
In this section, we will prove the properties
listed in Section~\ref{s:operations}.
First, we quote related lemmas for reference.
The proof of lemmas can be found in~\cite{dagnn, causalnf}.

\begin{lemma}
    \label{parent}
    In topological batching, for any \(b \in \mathbf{B}_j\) and the relative variable \(V_b\), we always have
    \(\mathbf{pa}_b \subset \mathop{\cup}\limits_{k = 1}^{j - 1}\mathbf{B}_k\)
\end{lemma}
\noindent
This lemma reveals the fact that the parent variables of \(V_b\) must appear in \(\mathbf{B}\) before \(V_b\).

\begin{theorem}
\label{cc}
    If a causal NF \(\mathcal{NF}\) with a transformation \(T_\theta\) \textbf{isolates} the exogenous variables \(\mathbf{U}\) of an SCM \(\mathcal{M}\),
    \(\mathcal{NF}\) is causally consistent with the true data-generating process of \(\mathcal{M}\).
\end{theorem}
\noindent
The word "isolates" in the theorem means that
the \(i\)-th output \(X_i\) only depends on \(U_i\),
i.e. \(X_i = T_\theta(U_i)\).

Since the partial causal transformation is a special kind of transformation, it has two important lemmas which will be used later.
\begin{lemma}[Composition]
    \label{comp}
    On any label set \(\mathbf{L}\), we can chain different partial causal transformations to form a large one, i.e. \(T_{\theta}^\mathbf{L} = T_{\theta_k}^\mathbf{L} \circ \cdots \circ T_{\theta_1}^ \mathbf{L}\).
\end{lemma}
\begin{lemma}[Partition]
    \label{part}
    For a large label set \(\mathbf{B}\) and its partition \((\mathbf{B}_1, \cdots, \mathbf{B}_n)\), 
    there always exists \(\theta, \theta_1, \cdots, \theta_n\) which can fulfill the equation: \(T_{\theta}^\mathbf{B} = T_{\theta_k}^{\mathbf{B}_n} \circ \cdots \circ T_{\theta_1}^{\mathbf{B}_1}\)
\end{lemma}

With all the lemmas, we can prove the properties of \sys as below.
\begin{theorem}[Causality]
    Given a \sys \(T_{\theta}^\mathbf{B}\),
    for the \(i\)-th variable \(X_i\),
    \(X_i\) only depends on its parents \(\mathbf{X}_{\mathbf{pa}_i}\) and \(U_i\).
    Particularly assume \(i \in \mathbf{B}_j\), we have \(X_i = T_{\theta_j}^\mathbf{B_j}(U_i \mid \mathbf{X}_{\mathbf{pa}_i})\)
\end{theorem}
\begin{proof}
    As \(\mathbf{B}\) is a partition of the label set, 
    for any label \(i\), it belongs to one and only one \(\mathbf{B}_j\).
    Therefore for any \(k \neq j\), \(T_{\theta_k}^\mathbf{B_k}\) performs identity transformation on the \(i\)-th variable
    by Equation~\ref{eq:pcaf}.
    i.e. we have 
    \begin{gather}
        U_i = Z^{0}_{i} = Z^{1}_{i} \cdots = Z^{j - 1}_{i} \\
        Z^{j}_{i} = Z^{j + 1}_{i} \cdots = Z^{n}_{i} = X_i\label{eq:two} \\
        X_i = Z^{j}_{i} = T_{\theta_j}^{\mathbf{B}_j}(Z^{j - 1}_{i} \mid \mathbf{Z}^{j - 1}_{\mathbf{pa}_i}) = T_{\theta_j}^{\mathbf{B}_j}(U_i \mid \mathbf{Z}^{j - 1}_{\mathbf{pa}_i}) \label{eq:three}
    \end{gather}

    From Equation~\ref{eq:two},
    we can infer that for any \(k < j\), we have \(\mathbf{Z}^{k}_{\mathbf{B}_k} = \cdots = \mathbf{Z}^{j - 1}_{\mathbf{B}_k} = \cdots = \mathbf{X}_{\mathbf{B}_k}\),
    and as shown in Lemma~\ref{parent}, \(\mathbf{pa}_i \subset \mathop{\cup}\limits_{k = 1}^{j - 1}\mathbf{B}_k\).
    Together we have \(\mathbf{Z}^{j - 1}_{\mathbf{pa}_i} = \mathbf{X}_{\mathbf{pa}_i}\).
    Putting this equation into Equation~\ref{eq:three}, we can infer that \(X_i = T_{\theta_j}^\mathbf{B_j}(U_i \mid \mathbf{X}_{\mathbf{pa}_i})\).
\end{proof}

\begin{theorem}[Universality]
    A \sys \(T_{\theta}^\mathbf{B}\) is a multi-layer universal approximator as long as for any \(j\), \(T_{\theta_j}^{\mathbf{B}_j}\) is a multi-layer universal approximator.
\end{theorem}
\begin{proof}
    Since we assume \(T_{\theta_j}^\mathbf{B_j}\) is a multi-layer universal approximator, the distribution \(P(\mathbf{X_{B_j}})\) can be approximated by a finite number of transformations chaining together: \(P(\mathbf{X_{B_j}}) = P(T_{\theta_{j_k}}^\mathbf{B_j} \cdots \circ T_{\theta_{j_1}}^\mathbf{B_j}(\mathbf{U_{B_j}}))\).
    Given Lemma~\ref{comp}, the chaining transformation can be expressed as \(T_{\theta_{j}}^\mathbf{B_j} = T_{\theta_{j_k}}^\mathbf{B_j} \cdots \circ T_{\theta_{j_1}}^\mathbf{B_j}\).
    Given Lemma~\ref{part}, we can chain all \(T_{\theta_j}^{\mathbf{B}_j}\) together to form \(T_{\theta}^\mathbf{B}\): \(T_{\theta}^\mathbf{B} = T_{\theta_n}^{\mathbf{B}_n} \circ \cdots \circ T_{\theta_1}^{\mathbf{B}_1}\).
    The \(T_{\theta}^\mathbf{B}\) is a multi-layer universal approximator
    since \(T_{\theta}^\mathbf{B}\) can approximate any \(P(\mathbf{X})\).
\end{proof}

\begin{theorem}[Causal Consistency]
    \(T_{\theta}^\mathbf{B}\) is causally consistent with the given SCM \(\mathcal{M}\).
\end{theorem}
\begin{proof}
    As shown in Theorem~\ref{causality},
    for any \(X_i\) and \(i \in \mathbf{B}_j\), \(X_i = T^\mathbf{B_j}(U_i)\).
    Here we omit the parameter \(\theta_j\) since it is fixed as long as the parent set \(\mathbf{X}_{\mathbf{pa}_i}\) is fixed.
    It indicates that \(X_i\) only depends on the value of \(U_i\).
    Given Theorem~\ref{cc}, we know that \(T_{\theta}^\mathbf{B}\) is causally consistent with \(\mathcal{M}\).
\end{proof}

\begin{theorem}[Minimum Layer]
    If the longest path of the DAG causal graph \(\mathcal{G}\) is \(d\), then \sys contains at least \(d\) layers. 
\end{theorem}
\begin{proof}
    Suppose that the longest path of the \(\mathcal{G}\) is \((V_1, \cdots, V_d)\), and we know that each node must belong to different \(\mathbf{B}_i\), i.e. \(V_i \in \mathbf{B}_i\).
    Therefore the partition \(\mathbf{B}\) contains at least \(d\) sets.
    Since each \(T_{\theta_i}^{\mathbf{B}_i}\) contains at least one layer, overall \sys contains at least \(d\) layers.
\end{proof}

As proved by prior works~\cite{vaca, causalnf}, if a generative model attempts to capture the complete causality, it must contain precisely one layer or a minimum of \(d\) layers.
However those models have no limitation on the number of layers,
therefore, unlike \sys, they cannot inherently ensure the completeness of the captured causality.
\tikzset{
  every matrix/.style={
            matrix of nodes, 
            nodes = {
                circle,
                minimum size = 0.5cm
            }, 
            row sep = 0.25cm,
            column sep = 0.5cm
            }
}

\section{Synthetic Dataset}
\label{appendix:synthetic}
In this section, we list all the SCMs used for creating synthetic datasets in this paper.
Any variable $U_i$ is sampled from $N(0, 1)$.

\PP{Non-linear Triangle}
The \textit{Non-linear Triangle} dataset consists of 3 endogenous variables.
The structural equations and the causal graph are list below. 
\begin{align*}
    X_0 &= U_0 + 1 \\
    X_1 &= 2 * X_{0}^{2} - Softsign(X_0 + U_1) \\
    X_2 &= 20 / (1 + e^{-X_{1}^{2} + X_0}) + U_2
\end{align*}

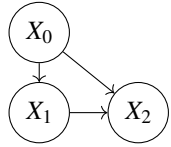
\begin{figure}[h]
    \centering
    \caption{Causal graph for Non-linear Triangle dataset}
    \begin{tikzpicture} 
        \matrix(m) {
            |[draw]| $X_0$ & ~ \\
            |[draw]| $X_1$ & |[draw]| $X_2$ \\
        };

        \draw[->] (m-1-1) -- (m-2-1);
        \draw[->] (m-1-1) -- (m-2-2);
        \draw[->] (m-2-1) -- (m-2-2);
    \end{tikzpicture}
\end{figure}

\PP{Non-linear Simpson}
The \textit{Non-linear Simpson} dataset consists of 4 endogenous variables, which simulates the Simpson's paradox.
The structural equations and the causal graph are list below. 
\begin{align*}
    X_0 &= U_0 \\
    X_1 &= Softsign(X_0) + \sqrt{3 / 20} * U_1 \\
    X_2 &= Softsign(2 * X_1) + 1.5 * X_0 - 1 +Softsign(U_2)\\
    X_3 &= (X_2 - 4) / 5 + 3 + 1 / \sqrt{10} * U_3
\end{align*}

\begin{figure}[h]
    \centering
    \caption{Causal graph for Non-linear Simpson dataset}
    \begin{tikzpicture} 
    
        \matrix(m) {
            |[draw]| $X_0$ & ~ & ~ \\
            ~ & |[draw]| $X_2$ & |[draw]| $X_3$ \\
            |[draw]| $X_1$ & ~ & ~ \\
        };

        \draw[->] (m-1-1) -- (m-3-1);
        \draw[->] (m-1-1) -- (m-2-2);
        \draw[->] (m-3-1) -- (m-2-2);
        \draw[->] (m-2-2) -- (m-2-3);
    \end{tikzpicture}
\end{figure}
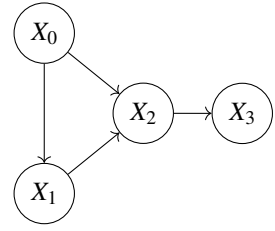

\PP{Four Node Chain}
The \textit{Four Node Chain} dataset consists of 4 endogenous variables.
This dataset is used in the causal consistency experiment.
The structural equations and the causal graph are list below. 
\begin{align*}
    X_0 &= U_0 \\
    X_1 &= 5 * X_0 - U_1 \\
    X_2 &= -0.5 * X_1 - 1.5 * U_2\\
    X_3 &= X_2 + U_3
\end{align*}

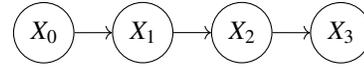
\begin{figure}[h]
    \centering
    \caption{Causal graph for Four Node Chain dataset}
    \begin{tikzpicture}
        \matrix (m) {
            |[draw]| $X_0$ & |[draw]| $X_1$ & |[draw]| $X_2$ & |[draw]| $X_3$ \\
        };
    
        \draw[->] (m-1-1) -- (m-1-2);
        \draw[->] (m-1-2) -- (m-1-3);
        \draw[->] (m-1-3) -- (m-1-4);
    \end{tikzpicture}
\end{figure}

\PP{M-Graph}
The \textit{M-Graph} dataset consists of 5 endogenous variables.
The structural equations and the causal graph are list below. 
\begin{align*}
    X_0 &= U_0 \\
    X_1 &= U_1 \\
    X_2 &= Softsign(e^{X_0} + U_2) \\
    X_3 &= (X_{1}^{2} + 0.5 * X_{0}^{2} + 1) * U_3 \\
    X_4 &= (-1.5 * X_{1}^{2} - 1) * U_4
\end{align*}

\begin{figure}[h]
    \centering
    \caption{Causal graph for Four M-Graph dataset}
    \begin{tikzpicture}
        \matrix (m) {
            ~ & |[draw]| $X_0$ & ~ & |[draw]| $X_1$ & ~ \\
            |[draw]| $X_2$ & ~ & |[draw]| $X_3$ & ~ & |[draw]| $X_4$ \\
        };
    
        \draw[->] (m-1-2) -- (m-2-1);
        \draw[->] (m-1-2) -- (m-2-3);
        \draw[->] (m-1-4) -- (m-2-3);
        \draw[->] (m-1-4) -- (m-2-5);
    \end{tikzpicture}
\end{figure}
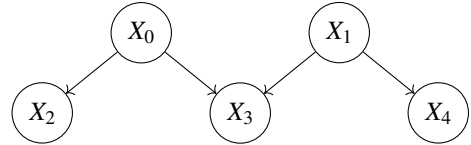

\PP{Network}
The \textit{Network} dataset consists of 6 endogenous variables.
It simulates a 2-layer neuron network with softplus activate layer.
The structural equations and the causal graph are list below. 
\begin{align*}
    X_0 &= U_0 \\
    X_1 &= Softplus(2 * X_0 + U_1 + 1)\\
    X_2 &= Softplus(-1.5 * X_0 + U_2 - 1) \\
    X_3 &= Softplus(2 * X_1 - 2.5 * X_2 + U_3) \\
    X_4 &= Softplus(0.5 * X_1 - 3.5 * X_2 + U_4) \\
    X_5 &= 10 * (Sigmoid(0.5 * X_3 - 1.5 * X_4 + U_5) - 0.5)
\end{align*}

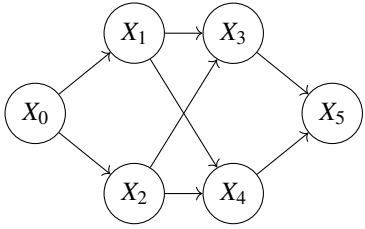
\begin{figure}[h]
    \centering
    \caption{Causal graph for Network dataset}
    \begin{tikzpicture}
        \matrix (m) {
            ~ & |[draw]| $X_1$ & |[draw]| $X_3$ & ~ \\
            |[draw]| $X_0$ & ~ & ~ & |[draw]| $X_5$ \\
            ~ & |[draw]| $X_2$ & |[draw]| $X_4$ & ~ \\
        };
    
        \draw[->] (m-2-1) -- (m-1-2);
        \draw[->] (m-2-1) -- (m-3-2);
        \draw[->] (m-1-2) -- (m-1-3);
        \draw[->] (m-1-2) -- (m-3-3);
        \draw[->] (m-3-2) -- (m-1-3);
        \draw[->] (m-3-2) -- (m-3-3);
        \draw[->] (m-1-3) -- (m-2-4);
        \draw[->] (m-3-3) -- (m-2-4);
    \end{tikzpicture}
\end{figure}

\PP{Backdoor}
The \textit{Backdoor} dataset consists of 7 endogenous variables.
It simulates a backdoor variable $X_0$ which controls the variable $X_6$ through two different paths.
The structural equations and the causal graph are list below. 
\begin{align*}
        X_0 &= U_0 \\
        X_1 &= Softplus(X_0 + U_1) \\
        X_2 &= Softplus(X_1 + U_2 + 2) \\
        X_3 &= Softplus(-X_0 + U_3) \\
        X_4 &= Softplus(X_2 + U_4 - 1) \\
        X_5 &= Softplus(X_3 + U_5 + 1) \\
        X_6 &= 10 * (Sigmoid(X_4 / 3 - X_5 / 3 +  U_6) - 0.5)
\end{align*}

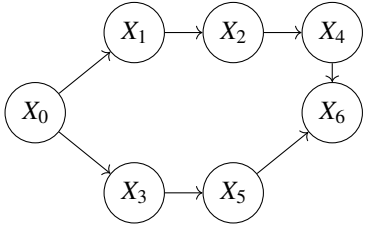
\begin{figure}[h]
    \centering
    \caption{Causal graph for Backdoor dataset}
    \begin{tikzpicture}
        \matrix (m) {
            ~              & |[draw]| $X_1$ & |[draw]| $X_2$ & |[draw]| $X_4$ \\
            |[draw]| $X_0$ & ~              & ~              & |[draw]| $X_6$ \\
            ~              & |[draw]| $X_3$ & |[draw]| $X_5$ & ~              \\
        };
    
        \draw[->] (m-2-1) -- (m-1-2);
        \draw[->] (m-2-1) -- (m-3-2);
        \draw[->] (m-1-2) -- (m-1-3);
        \draw[->] (m-1-3) -- (m-1-4);
        \draw[->] (m-1-4) -- (m-2-4);
        \draw[->] (m-3-2) -- (m-3-3);
        \draw[->] (m-3-3) -- (m-2-4);
    \end{tikzpicture}
\end{figure}

\PP{Eight Node Chain}
The \textit{Eight Node Chain} dataset consists of 8 endogenous variables.
This dataset is used in the causal inference tasks experiment.
The structural equations and the causal graph are list below. 

\begin{align*}
        X_0 &= U_0 \\
        X_1 &= e^{Softsign(x0 + U_1)} \\
        X_2 &= X_1 + Softsign(e^{X_1} * U_2) \\
        X_3 &= X_2 + Softsign(e^{X_2} * U_3) \\
        X_4 &= X_3 + Softsign(e^{X_3} * U_4) \\
        X_5 &= X_4 + Softsign(e^{X_4} * U_5) \\
        X_6 &= X_5 + Softsign(e^{X_5} * U_6) \\
        X_7 &= X_6 + Softsign(e^{X_6} * U_7)
\end{align*}

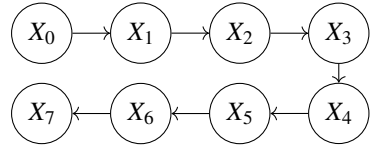
\begin{figure}[h]
    \centering
    \caption{Causal graph for Eight Node Chain dataset}
    \begin{tikzpicture}
        \matrix (m) {
            |[draw]| $X_0$ & |[draw]| $X_1$ & |[draw]| $X_2$ & |[draw]| $X_3$ \\
            |[draw]| $X_7$ & |[draw]| $X_6$ & |[draw]| $X_5$ & |[draw]| $X_4$ \\
        };
    
        \draw[->] (m-1-1) -- (m-1-2);
        \draw[->] (m-1-2) -- (m-1-3);
        \draw[->] (m-1-3) -- (m-1-4);
        \draw[->] (m-1-4) -- (m-2-4);
        \draw[->] (m-2-4) -- (m-2-3);
        \draw[->] (m-2-3) -- (m-2-2);
        \draw[->] (m-2-2) -- (m-2-1);
    \end{tikzpicture}
\end{figure}

\section{Evaluation}
\label{appendix:eval}
Here are the technical details of the experiments in Section~\ref{s:eval}.

\begin{figure*}[h]
    \centering
    \subfloat[Intervened distributions of $X_1$]{
        \includegraphics[width=2.5in]{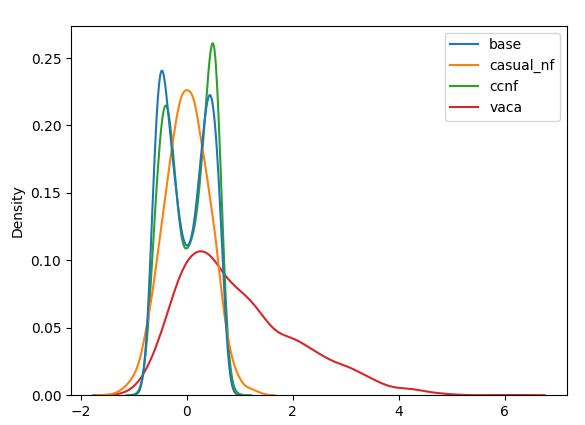}
         \label{fig:intervene-1}
    }
    \subfloat[Intervened distributions of $X_2$]{
        \includegraphics[width=2.5in]{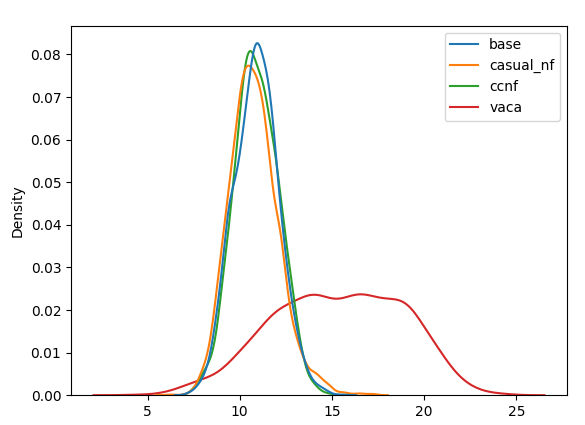}
        \label{fig:intervene-2}
    }
    \caption{Intervened distributions of $Nlin-Triangle$ on \sys, VACA, causalNF.}
    \label{fig:intervene}
\end{figure*}

\PP{Setup}
All experiments are conducted on an 8-core AMD Ryzen 7 6800H CPU, 16G RAM, and Geforce RTX 3070 Ti GPU. 
Every experiment adopts PyTorch 2.1.0+cu121 and CUDA 12.0 as the base framework.

\subsection{Causal Consistency}

\PP{Experiment design}
For this experiment, 
we utilize a dataset of i.i.d observations generated from a four-line-chain SCM,
details of which are presented in Appendix~\ref{appendix:synthetic}.
The length of the longest path, 
denoted as \(d\) in Section~\ref{s:operations},
is 4.
As the SCM exhibits a sparse Jacobian matrix, 
any causal inconsistency is readily apparent,
as highlighted in previous studies~\cite{causalnf}.

In our evaluation, all models except \sys are categorized into two types: a model with one layer ($L = 1$) and a model with more than one layer ($L > 1$).
For VACA, $L$ refers to the number of layers in the encoder architecture.
For the decoder architecture of VACA, the number of layers is 4
to maintain causal consistency with the SCM.
\sys is impossible to be implemented with one layer ($L = 1$) since \sys requires at least 4 layers in the four-line-chain dataset according to Theorem~\ref{mini},
therefore we only implement \sys with more than one layer ($L > 1$) for this experiment.

\PP{Hyperparameters selection}
The selection ranges for hyperparameters 
are derived from the own and related works
of those models,
ensuring relevance and comparability. 
To ensure a fair comparison, 
we adjust certain hyperparameter selections
so that all models are evaluated
under similar conditions.
We employ the Optuna framework~\cite{optuna} to 
systematically cross-validate and 
identify the optimal hyperparameter settings
for each model 
involved in our experiment.

\emph{1) CAREFL}: we adopt the MAF with ReLU~\cite{relu} activation layers.
The hidden layers of each MAF are chosen from the following combinations: \{[8, 8], [16, 16], [32, 32], [64, 64]\}.

\emph{2) VACA}: the GNN architecture is chosen from \{PNA, GIN, GAT\}.
The number of hidden channels inside each GNN is chosen from \{16, 32\}.
The number of layers in the encoder is fixed at L, and the number in the decoder equals \(max(4, d)\).
The number of layers in the MLP before the GNN is chosen from \{0, 1\}.

\emph{3) CausalNF}: the structure of NFs is chosen from \{MAF, flow++~\cite{maf}\} with the ReLU activation layer.
The number of hidden layers of each NF is chosen from \{[8, 8], [16, 16], [32, 32], [64, 64]\}.

\emph{4) \sys}: we also choose the NFs from \{MAF, flow++\} with the ReLU activation layer.
The number of hidden layers is also chosen from the same combinations as others: \{[8, 8], [16, 16], [32, 32], [64, 64]\}.
However, unlike others, the length of each \(T_{\theta}^{\mathbf{B}}\) in \sys varies from 1 to 5.

\PP{Training methodology}
The synthetic dataset is generated through the following process:
We initially generate 30000 samples using the underlying SCM,
and randomly choose 25000 samples for training, 2500 for evaluation, and 2500 for testing.
For training the models, we establish consistent parameters and methodologies across all experiments.
Every model is trained with a maximum of 1000 epochs, and training halts when the validation loss has not decreased for 50 consecutive epochs.
We utilize Adam~\cite{adam} with a learning rate set to 0.001.
To mitigate randomness, every experiment is conducted five times and the average results and the maximum deviation from the average are recorded for analysis.

Note that in the subsequent experiment on causal inference tasks,
we maintain the same hyperparameters and training methodology
as in this experiment.
This is reasonable because
the models trained in the first experiment
can be used directly in the second one.

\subsection{Causal Inference Tasks}

\PP{Measurement}
According to~\cite{pearl}, there exists three types of causality and each corresponds to different causal inference tasks: \textit{observations},
\textit{interventions}, and
\textit{counterfactuals}.
Therefore we use three different measurements to evaluate the performance of causal inference tasks.

\emph{1) For observations}, the measurement is the KL distance between prior distributions captured by generative models and the normal distribution.

\emph{2) For interventions}, we measure the max \textit{Maximum Mean Discrepancy} (MMD) distance between the intervened distributions given by generative models and the actual intervened distributions.

\emph{3) For counterfactuals}, we measure the \textit{Root-Mean-Square Deviation} (RMSD) distance between the actual counterfactuals and the generated counterfactuals.

\PP{Case Study}
To make a more intuitive comparison, we conduct a case study on the nlin-triangle dataset, whose underlying generative model is a three-node-triangle SCM.
Equations are listed here and in Appendix~\ref{appendix:synthetic}.
\begin{align*}
    X_0 &= U_0 + 1 \\
    X_1 &= 2 * X_{0}^{2} - Softsign(X_0 + U_1) \\
    X_2 &= 20 / (1 + e^{-X_{1}^{2} + X_0}) + U_2
\end{align*}
here, the equation between $X_1$ and $U_1$ is non-linear, while the one between $X_2$ and $U_2$ is linear.
In this study, we perform $Do(X_0 = 0)$ and record the distribution of $X_1$ and $X_2$ on each model.
Results are shown in Figure~\ref{fig:intervene}.
As seen from the graph, compared with the actual distributions (blue),
\sys (green) can successfully simulate both distributions of $X_1$ and $X_2$,
CausalNF (orange) can only simulate the distribution of $X_2$ and
VACA (red) fails to capture both distributions.

These outcomes corroborate the theoretical expectations Figure~\ref{s:operations}.
Since \sys is a multi-layer universal approximator, it possesses the capability to capture the true generative process regardless of its complexity,
hence performing well on both intervened distributions.
CausalNF limits its depth for maintaining causal consistency,
resulting in the inability to capture complex relationships like $X_1$.
VACA, leveraging GNN for encoding and decoding the SCM, therefore is susceptible to over-parameterization issues.
This problem eventually leads VACA to deviations from the actual distributions.

\subsection{Real-world Evaluation}

\PP{Hyperparameter tuning and training method}
As we only train the German credit database on \sys.
We manually explore a range of hyperparameters to identify the best ones.
We opt for flow++ for all \(T_{\theta}^{\mathbf{B}}\) instances.
Regarding the number of hidden layers in each NF, we set it to 3 for all layers except 5 for the last one.
Similar to the prior experiment, we train \sys by Adam optimization in 1000 epochs with a learning rate of 0.01.
We repeat the training process five times and record the average result and the maximum deviation from the average.

\PP{ITE and ATE Evaluation}
ITE and ATE are two methods to evaluate 
the result of causal inference in semi-supervised datasets.
ITE stands for the difference in outcomes 
for an individual $i$ under treatment $T$:
\begin{equation*}
    ITE_i = Y_{i}(T = 1) - Y_{i}(T = 0)
\end{equation*}
while ATE stands for the average of the
ITE across all individuals:
\begin{equation*}
    ATE = \sum_{i=1}^{n} ITE_i
\end{equation*}
In this experiment, we use \sys to measure 
the ATE value of $age$ and $gender$ in Table \ref{tab:ate}:

\begin{table}[h]
    \centering
    \begin{tabular}{lll}
\toprule

Name &
$ATE_{gender}$ &
$ATE_{age}$ \\
\midrule

CCNF &
0.00$\pm$0.00 &
0.00$\pm$0.00 \\
\bottomrule
\end{tabular}
    \caption{ATE evaluation on the German credit dataset}
    \label{tab:ate}
\end{table}

We notice that $ATE_{gender}$ and
$ATE_{age}$ are zero,
meaning that both values do not have
any effect on the $risk$.
The result is not surprising: 
since the $age$ and $gender$ are not directly used 
when computing the $risk$,
changing the value directly 
cannot affect the value of $risk$.

\end{document}